\documentclass[final,3p,times]{elsarticle}
\usepackage{amssymb}
\usepackage{graphicx,setspace}
\usepackage{psfrag}
\usepackage{amsfonts}
\expandafter\let\csname equation*\endcsname\relax
\expandafter\let\csname endequation*\endcsname\relax
\usepackage{amsmath}
\usepackage{amssymb, amsthm}
\usepackage{mathrsfs}
\usepackage{epstopdf}
\usepackage{float}
\usepackage{lineno,hyperref}
\usepackage{color}
\usepackage{subfigure}
\usepackage{booktabs}

\usepackage{amsmath}
\usepackage{dsfont} % it is the package of using the "\mathds{} "
\usepackage{amssymb} % it is the package of using the "\mathbb{}"
\usepackage{graphicx,color}
\usepackage{extarrows}

\usepackage{graphicx}

\usepackage{epstopdf}
\journal{Nonlinear Dynamics}

\begin{document}
\newtheorem{definition}{Definition}[section]
\newtheorem{lemma}{Lemma}[section]
\newtheorem{remark}{Remark}[section]
\newtheorem{theorem}{Theorem}[section]
\newtheorem{proposition}{Proposition}
\newtheorem{assumption}{Assumption}
\newtheorem{example}{Example}
\newtheorem{corollary}{Corollary}[section]
\def\ep{\varepsilon}
\def\Rn{\mathbb{R}^{n}}
\def\Rm{\mathbb{R}^{m}}
\def\E{\mathbb{E}}
\def\hte{\hat\theta}
%\numberwithin{theorem}{section}
%\numberwithin{definition}{section}
\renewcommand{\theequation}{\thesection.\arabic{equation}}
\begin{frontmatter}



\title{Controlling mean exit time of stochastic dynamical systems based on quasipotential and machine learning}

\author[addr1]{Yang Li}\ead{liyangbx5433@163.com}
\author[addr2]{Shenglan Yuan\corref{cor1}}\ead{shenglan.yuan@math.uni-augsburg.de}\cortext[cor1]{Corresponding author}
\author[addr1]{Shengyuan Xu}\ead{syxu@njust.edu.cn}

\address[addr1]{School of Automation, Nanjing University of Science and Technology, Nanjing 210094, China}
\address[addr2]{Institut f\"{u}r Mathematik, Universit\"{a}t Augsburg,
86135, Augsburg, Germany}

\begin{abstract}
The mean exit time escaping basin of attraction in the presence of white noise is of practical importance in various scientific fields.  In this work, we propose a strategy to control mean exit time of general stochastic dynamical systems to achieve a desired value based on the quasipotential concept and machine learning. Specifically, we develop a neural network architecture to compute the global quasipotential function. Then we design a systematic iterated numerical algorithm to calculate the controller for a given mean exit time. Moreover, we identify the most probable path between metastable attractors with help of the effective Hamilton-Jacobi scheme and the trained neural network. Numercal experiments demonstrate that our control strategy is effective and sufficiently accurate.
\end{abstract}


\end{frontmatter}


\section{Introduction}
Throughout history, scientists have attempted to model practical systems using mathematical equations. This has been quite successful in some scientific fields, but not in all. Stochastic dynamical systems have become an important modeling tool in numerous areas of science and engineering.  For example, ecosystems \cite{QY,YLZ}, Hamiltonian mechanics \cite{YB}, fluid dynamics \cite{YBD}, biological neurons \cite{YZD},  physical applications \cite{ZYHD} and many others \cite{CWY,TTYBD,YHLD} are often modeled using stochastic dynamical systems. Mathematical modeling of dynamical systems under
uncertainty often leads to stochastic differential equations (SDEs); see Protter \cite{P}.

When the noise intensity in multistable equilibrium system is non-vanishing, the stochastic effects induce a stochastic dynamics which destabilizes those attractors and leads
to random transitions between its coexisting basins of attraction \cite{HTB}. This situation is analogous to an escape dynamics from metastable wells \cite{FM}. There are many possible escape/switching paths, but there is a path along which switching or escape is most likely to occur. Characterization of the most probable transition path of stochastic dynamical system enables the determination of the mean time to escape from a metastable state to another metastable state \cite{HCYD}. The effect of external noise is often described using a Langevin equation (a master equation) or the associated Fokker-Planck equation. We can formulate a variational problem, either analytically or numerically, to find the optimal path of escape or switching that ultimately reduces to considering trajectories of an auxiliary Hamiltonian dynamical system \cite{LDLZ}.

The mean exit time can measure the escape over a potential barrier of the vector field, i.e., drift. The trajectories from the left hand well to the right well, take most of the time actually surmounting the barrier. We want to evaluate the mean exit time in going over the barrier for
analysis of dynamical systems perturbed by small noise. It is quite meaningful to talk of the escape time at that time for the trajectory, initially at stable point, to reach a point near saddle, since this time is quite insensitive to the exact location of the initial and final points.

For Freidlin's and Wentzell's large deviation theory for the nongradient SDEs with small white noise, the quasipotential is a key concept \cite{FW}.
In a recent development, There are novel and exciting applications in cardiology, population dynamics, communications, engineering,
laser technology, space research, and genetic switches. Aurell and Sneppen \cite{AS} examined the problem of escape from a stable equilibrium and its bifurcation with changing parameter in more than one dimension, and demonstrated how this determines the stability and the robustness of states of genetic networks. Using a nontrivial adjustment of the Ordered Upwind Method, Cameron computed the quasipotential on a mesh to find the most probable paths by numerical integration \cite{C2012}. Chen, Zhu and Liu \cite{CZL} considered the noise-induced escapes in an excitable system possessing a quasi-threshold manifold, along which there exists a certain point of minimal quasipotential. Dahiya and Cameron \cite{DC} performed numerical computation of the quasipotential for SDEs with multiplicative noise on a mesh, and discussed an application to the Maier-Stein model with anisotropic diffusion. Lv et al. \cite{LLLL} characterized the metastability of gene regulatory system perturbed by intrinsic noise, and constructed the global quasipotential energy landscape to calculate the optimal transition paths between the on and off states based on the large deviation theory. Nolting and Abbott \cite{NA} visualized stable states in stochastic systems using a ball-in-cup diagram, and  provided the quasipotential as a practical tool to quantify stability in stochastic systems.
Yang, Potter and Cameron \cite{YPC} analyzed the ordered line integral method with midpoint rule for finding the quasipotential in 3D by use of Karush-Kuhn-Tucker theory for rejecting unnecessary simplex updates, and conducted an upgraded hierarchical update strategy to prune the number of admissible simplexes and a fast search for them.

A neural network is a powerful parallel information-processing system and universal approximation map \cite{LD}. It
is composed of neurons and searches for the minimum of the loss function during the learning process by gradient descent method.
Neural networks are being used extensively in the fields of aeronautics,
robotics, defense, engineering, telecommunications, manufacturing, medicine,
insurance, psychology, banking, security, marketing, and finance \cite{S1996}. These artificial networks may be used for predictive modeling, adaptive control and applications where they can be trained via a dataset \cite{LDL}. Self-learning resulting from experience can occur within networks, which can derive conclusions from a complex and seemingly unrelated set of information \cite{T}. It is worth pointing out that machine learning provides
a means of tackling nondeterministic systems, where the equations used to
model the system are not known \cite{LXDLC}.

Our goal is to control the mean exit time of stochastic dynamical systems to achieve a desired value based on the
quasipotential and machine learning. It visualizes the dynamics, and allows us to find most probable transition
paths and estimate stationary probability density.

This paper is organized as follows. In Section \ref{Q}, we make a portrait of the quasipotential as the potential component of orthogonal decomposition for the nongradient SDEs with a finite number of isolated attractors by decomposing a smooth nongradient vector field  to a potential component and a rotational component in each basin of attraction. In Section \ref{C}, we give a description of the control strategy and devise a new technique to control the mean exit time of stochastic dynamical systems by utilizing the quasipotential and neural network. In Section \ref{N}, we perform numerical experiments to demonstrate the effectiveness of our control strategy by computing the quasipotential and minimum action paths in stochastic Maier-Stein model having two distinct stable points SN1 and SN2 separated by the potential barrier at US. In Section \ref{CF}, We draw the conclusions from the discussions and summarize the future challenges.

\section{Quasipotential}\label{Q}
We consider the two-dimensional physical model described by stochastic differential equation (SDE)
\begin{equation}\label{SM}
\dot{x}(t)=F(x(t))+\eta(t),
\end{equation}
where $F: \mathbb{R}^{2}\rightarrow\mathbb{R}^{2}$ represents a continuously differentiable vector field, and $\eta(t)$ can often be regarded as an additive randomly distributed noise, perturbing the state $x(t)$. In a formal sense, $\eta_{i}(t)$, $i=1,2$ can be equated with the time derivative $\partial_tB_{i}$ of a Brownian motion $B_{i}(t)$, resulting in paths $x_{i}(t)$ that are continuous everywhere but differentiable nowhere.

The stochastic Euler scheme is
\begin{equation*}
\Delta x=F(x(t))\Delta t+\Delta\eta,
\end{equation*}
where $\Delta x=x(t+\Delta t)-x(t)$. It is thoughtful of us to investigate
the statistics of finite increments $\Delta\eta_{i}=B_{i}(t+\Delta t)-B_{i}(t)$, $i=1,2$. Those real-valued random variables are normally distributed, whose Gaussian distribution has finite mean and variance determined by
\begin{equation*}
\mathbb{E}[\Delta\eta_{i}]=0\quad\quad\text{and}\quad\quad\mathbb{E}[(\Delta\eta_{i})^{2}]=\sigma\Delta t,
\end{equation*}
where $\sigma$ is the noise intensity. Now we use the Euler--Maruyama
method to generate a random walk approximating a diffusion
process expressed in the SDE \eqref{SM}. When $\Delta t\rightarrow0$, we have ($i,j=1,2$)
\begin{equation*}
\mathbb{E}[\eta_{i}(t)]=0\quad\quad\text{and}\quad\quad\mathbb{E}[\eta_{i}(t)\eta_{j}(s)]=\sigma\delta_{ij}\delta(t-s),
\end{equation*}
where $\delta$ is the Dirac delta function, and $\delta_{ij}$ denotes the Kronecker delta.

The probability density for the Gaussian noise $\eta(t)$ can be expressed as follows:
\begin{equation*}
\mathbb{P}[\eta(t)]\varpropto\text{exp}\Big(-\frac{1}{2\sigma}\int_{0}^{T} \left| \eta \right| ^{2}dt\Big),
\end{equation*}
where $\varpropto$ denotes the logarithmic equivalence relation. Since $\eta=\dot{x}-F(x)$, we get
\begin{equation*}
\mathbb{P}[x(t)]\varpropto\text{exp}\Big(-\frac{1}{2\sigma}\int_{0}^{T} \left| \dot{x}-F(x) \right| ^{2}dt\Big).
\end{equation*}
The Freidlin-Wentzell action functional \cite{FW} of $x(t), t\in[0,T],$ is the line integral
\begin{equation}\label{AF}
\mathcal{S}_{T}[x(t)]=\int_{0}^{T}L(x,\dot{x})dt,
\end{equation}
where the Lagrangian is given by $L(x,\dot{x})=\frac{1}{2}\left|\dot{x}-F(x)\right|^{2}$.
Minimization w.r.t. paths and time can be done analytically leading to the minimum action \cite{HV}.
Performing this minimization over paths from points $x_0$ to $\bar{x}$
yields the quasipotential
\begin{equation}\label{QP}
V(\bar{x}):=\inf_{T>0}\inf_{x\in C[0,T]}\Big\{\,\mathcal{S}_{T}[x]: x(0)=x_0, x(T)=\bar{x}\,\Big\},\quad V(x_0)=0,
\end{equation}
which characterizes the difficulty of the random state fluctuating to the point $\bar{x}$.

\begin{remark}\label{remark1}
The quasipotential function $V(x)$ gives asymptotic estimates for the stationary probability density within the basin of attractors of
$\dot{x}=F(x)$ in the limit $\sigma\rightarrow0$. Moreover, the mean exit time $T_E$ exponentially depends on the minimal value $V_0$ of the quasipotential on the boundary of basin domain $D$ (generally the transition state $\bar{x}$ located at the saddle point), i.e.,
\begin{equation}\label{MET}
	T_{E}=be^{V_0/\sigma}, \ \ \  V_0 := \inf_{x \in \partial D} V(x),
\end{equation}
where $T_{E} := \mathbb{E} \tau_{\text{\rm exit}}$ and the exit time $\tau_{\text{\rm exit}} :=\inf \{ t >0 \ | x(0)=x_0, x(t)\notin D \}$.
\end{remark}

According to
\begin{equation*}
	\mathbb{P}[x(t)]\varpropto \text{exp} \Big(-\frac{\mathcal{S}_{T}[x(t)]}{\sigma}\Big),
\end{equation*}
finding the most probable path is transformed into the minimization problem of the action functional. For $\sigma$ sufficiently small, the optimal path $x_{\text{opt}}$ (solution of this minimization problem) can be obtained from \eqref{AF} using the
calculus of variations. The optimal path $x_{\text{opt}}$ maximizes $\mathbb{P}[x(t)]$ by minimizing $\mathcal{S}_{T}[x(t)]$. Upon demanding that the first variation $\delta \mathcal{S}_{T}[x]=0$, we get
\begin{align*}
\delta \mathcal{S}_{T}[x]&=\mathcal{S}_{T}[x+\delta x]-\mathcal{S}_{T}[x]=\int_{0}^{T}\mathcal{L}(x+\delta x,\dot{x}+\delta\dot{x})dt-\int_{0}^{T}\mathcal{L}(x,\dot{x})dt\\
                 &=\int_{0}^{T}\Big(\frac{\partial \mathcal{L}}{\partial x}\cdot\delta x+\frac{\partial \mathcal{L}}{\partial \dot{x}}\cdot\delta\dot{x}\Big)dt=\int_{0}^{T}\Big[\frac{\partial \mathcal{L}}{\partial x}-\frac{d}{dt}\Big(\frac{\partial \mathcal{L}}{\partial \dot{x}}\Big)\Big]\cdot\delta x dt,
\end{align*}
where we expanded $\mathcal{L}$ in a Taylor series, and integrated
by parts to isolate $\delta x$. Since the variation $\delta x $ is arbitrary,  $\delta \mathcal{S}_{T}[x]$ vanishes if and only if
\begin{equation*}
\frac{\partial \mathcal{L}}{\partial x}-\frac{d}{dt}\Big(\frac{\partial \mathcal{L}}{\partial \dot{x}}\Big)=\ddot{x}-\big(\nabla F(x)\big)^{\top}F(x)=0.
\end{equation*}
Here and in the following, $\nabla F(x)$ represents the Jacobian matrix with components
$(\nabla F(x))_{ij}=\partial F_i(x)/\partial x_j$, $i,j=1,2$.
So that $x_{\text{opt}}$ is a solution of second-order Euler-Lagrange differential equation:
\begin{equation*}
\ddot{x}=\big(\nabla F(x)\big)^{\top}F(x),\quad x(0)=x_0,\quad x(T)=\bar{x}.
\end{equation*}
Note that this is a boundary value problem and not a traditional initial value
problem. The dimensions of the Euler--Lagrange equation are twice the dimensions of the original SDE system \eqref{SM}.

The conjugate momentum $p$ is defined as
\begin{equation*}
p:=\frac{\partial \mathcal{L}}{\partial \dot{x}}=\dot{x}-F(x).
\end{equation*}
Then, $\dot{x}=F(x)+p$,  illustrates the connection between the optimal fluctuation $\eta$ and the classical momentum $p$. The Legendre transformation from the Lagrangian $L$ to the Hamiltonian $H$ is characterized by
\begin{equation*}
H(x,p)=p\cdot\dot{x}-L(x,\dot{x})=\big(\dot{x}-F(x)\big)\cdot\dot{x}-\frac{1}{2}\big(\dot{x}-F(x)\big)^{2}=\frac{1}{2}\big(\dot{x}-F(x)\big)^{2}+\big(\dot{x}-F(x)\big)\cdot F(x)=\frac{1}{2}p^{2}+p\cdot F(x).
\end{equation*}
The inverse of the Legendre transformation \cite{YD} can be used to obtain
the Lagrangian from the Hamiltonian, i.e., $L(x,\dot{x})=p\cdot\dot{x}-H(x,p)$.
\begin{lemma} (Hamiltonian formulation)
If $x(t)$ is a stationary point of the action functional \eqref{AF}, then it satisfies Hamilton's equations.
\end{lemma}
\begin{proof}
Rewriting the action functional in \eqref{AF}  into
\begin{equation*}
\mathcal{S}_{T}[x,p]=\int_{0}^{T}\big(p\cdot\dot{x}-H(x,p)\big)dt.
\end{equation*}
Since $\mathcal{S}_{T}$ is stationary, $\delta\mathcal{S}_{T}$ must vanish. Consequently,
\begin{align*}
\delta\mathcal{S}_{T}[x,p]&=\mathcal{S}_{T}[x+\delta x,p+\delta p]-\mathcal{S}_{T}[x,p]\\
                          &=\int_{0}^{T}[(p+\delta p)\cdot(\dot{x}+\delta\dot{x})-H(x+\delta x,p+\delta p)]dt-\int_{0}^{T}\big(p\cdot\dot{x}-H(x,p)\big)dt\\
                          &=\int_{0}^{T}\Big(\delta p\cdot\dot{x}+p\cdot\delta\dot{x}-\frac{\partial H}{\partial x}\cdot\delta x-\frac{\partial H}{\partial p}\cdot\delta p\Big)dt\\
                          &=\int_{0}^{T}\Big[\Big(\dot{x}-\frac{\partial H}{\partial p}\Big)\cdot\delta p-\Big(\dot{p}+\frac{\partial H}{\partial x}\Big)\cdot\delta x\Big]dt=0,
\end{align*}
where the third equality holds by using the expansion and keeping only those that are of the first order in
the small quantities $\delta x$ and $\delta p$. The fourth equality holds by isolating $\delta x$, integrating the term $p\cdot\delta\dot{x}$ by parts, and rearranging these terms. For $p\cdot\delta\dot{x}=\frac{d}{dt}(p\cdot\delta x)-\dot{p}\cdot\delta x$, the
integral of the total derivative term vanishes since $\delta x$ is zero vector at the endpoints, i.e.,
\begin{equation*}
\int_{0}^{T}\frac{d}{dt}(p\cdot\delta x)dt=p(T)\cdot\delta x(T)-p(0)\cdot\delta x(0)=p(T)\cdot\delta \bar{x}-p(0)\cdot\delta x_0=0.
\end{equation*}
Because $\delta x$ and $\delta p$ are arbitrary and
independent functionals, and the integrand is continuous, each of the
parenthesized terms above must vanish in order to obtain the last equality. Therefore, the path $x(t)$ satisfies Hamilton's equations
\begin{equation}\label{HamiltonE}
  \left\{
    \begin{array}{ll}
      \dot{x}=\frac{\partial H}{\partial p}=p+F(x), &  \\[1ex]
      \dot{p}=-\frac{\partial H}{\partial x}=-(\nabla F)^{\top}p. &
    \end{array}
  \right.
\end{equation}
\end{proof}

In addition, there is an another method, the so-called WKB approximation \cite{KS}, to derive the Hamiltonian formulation. Its kernel idea is to express the stationary or quasi-stationary distribution as the WKB form
\begin{equation}
	\label{WKB}
	p_{s} (x) \sim C(x) \text{exp} \left\{ -\frac{ V(x)}{\sigma} \right\},
\end{equation}
where $V(x)$ is the quasipotential defined in Eq. (\ref{QP}), and $C(x)$ is an exponential prefactor. Substituting it into stationary Fokker-Planck equation $\frac{1}{2}\Delta p(x)-\nabla p(x)\cdot F(x)=0$ and collecting the lowest-order terms of $\sigma$ yield a nonlinear partial differential equation called Hamilton-Jacobi equation
\begin{equation}
	\label{HJB}
	H\left(x, \nabla V(x) \right) := \big\langle \nabla V(x), F(x) \big\rangle + \frac{1}{2} \big\langle \nabla V(x), \nabla V(x) \big\rangle=0.
\end{equation}
There exists Ordered Upwind Method for solving Hamilton-Jacobi equations; see Sethian and Vladimirsky \cite{SV}.
\begin{remark}
The quasipotential is one of the solutions for \eqref{HJB}. But it is never unique! A simple observation shows that Eq. (\ref{HJB}) always has a trivial solution $V=0$.
\end{remark}
Note that there exists a geometric meaning in Eq. (\ref{HJB}), i.e., $\nabla V(x)$  is perpendicular
to $F(x) + \frac{1}{2} \nabla V(x)$. Assuming $l(x):=\frac{1}{2} \nabla V(x)+F(x)$ leads to an orthogonal decomposition of the vector field,
\begin{equation}\label{F}
F(x)=-\frac{1}{2}\nabla V(x)+\Big(\frac{1}{2}\nabla V(x)+F(x)\Big)\overset{\triangle}{=}-\frac{1}{2}\nabla V(x)+l(x).
\end{equation}
\begin{remark}
The orthogonal decomposition of the vector field can be seen as the rotation component $l(x)$ plus the potential component $-\frac{1}{2}\nabla V(x)$ (derived from the quasipotential gradient). If $l(x)\equiv0$, then the vector field $F(x)=-\frac{1}{2}\nabla V(x)$  is reduced as a gradient field.
\end{remark}

The quasipotential can be calculated via applying method of characteristics to the Hamilton-Jacobi equation (\ref{HJB}) and integrated by
\begin{equation}
	\label{Vdot}
	\dot{V}=p \cdot \dot{x} = \frac{1}{2} p^{T}p
\end{equation}
along with the parameterized trjectories of \eqref{HamiltonE}, where the momentum $p=\nabla V(x)$. If we know the quasipotential $V(x)$, then the most probable path can be computed by
\begin{equation}\label{xdot}
\dot{x}=F(x)+p=F(x)+\nabla V(x)=\frac{1}{2}\nabla V(x)+l(x).
\end{equation}
In other words, the minimum action paths are determined by $\|\dot{x}\|=\|\frac{1}{2}\nabla V(x)+l(x)\|$. Denote $G(x):=\frac{1}{2}\nabla V(x)+l(x)$, $\dot{x}$ is roughly parallel to $G(x)$.

Suppose $l(x)$ is not identically zero in \eqref{F}. As shown in Fig. \ref{fig1}, $l(x)$ is orthogonal to $\nabla V(x)$, where $F(x)$ is a smooth nongradient vector field with a finite number of isolated attractors. The situation is more complicated. The effect of noise is equivalent to keeping the rotation component unchanged, but flipping the inward potential component outward, so that it can consume the least energy to push the system out of the basin of attraction. Note that the estimates for the invariant probability measure, most probable transition paths, and mean exit times from basins of attractors, rely on the computation of quasipotential, according to Remark \ref{remark1} and Eq. (\ref{xdot}). Unfortunately, the decomposition of $F(x)$ can be done analytically only in special cases. Hence, we develop deep learning methods to compute the quasipotential based on Hamilton-Jacobi equation and then devise a control strategy to control mean exit time to achieve a desired value.

\begin{figure}
	\centering
	\includegraphics[width=7cm]{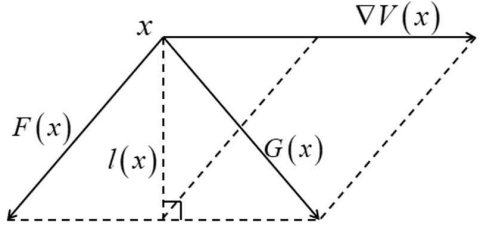}
	\caption{Orthogonal decomposition of vector field.}
	\label{fig1}
\end{figure}

\section{Control strategy}\label{C}
Thus far, as a powerful tool of machine learning or deep learning, the field of neural networks has generated a phenomenal
amount of interest from a broad range of scientific disciplines. One of the
reasons for this is adaptability. Innovative architectures and new training
rules have been tested on powerful computers. The disciplines of networks and nonlinear dynamics have increasingly coalesced.
The vast majority of real-world applications have relied on the backpropagation
algorithm for training multilayer networks, and recently kernel machines have
proved to be useful for a wide range of applications, including document
classification, gene analysis, pattern recognition, computer vision, credit card fraud, prediction and forecasting, disease recognition, facial
and speech recognition, the consumer home entertainment market, psychological profiling, predicting wave over-topping events, data-driven modelling, and control problems.

The aim of this section is to show that the system \eqref{SM} can achieve the desired mean exit time as we wish after adding the control term. To gain the desired mean exit time, we set
\begin{equation}\label{CS}
\dot{x}=u+F(x)+\eta.
\end{equation}
The control $u$ is designed to be proportional to the gradient of quasipotential. We choose $u=c\nabla V(x)$, where $c\leq\frac{1}{2}$ are being used as controllers. According to the orthogonal decomposition of the vector field (\ref{F}), the new quasipotential of Eq. (\ref{CS}) is controlled by the parameter $c$. Thus the desired mean exit time can be conveniently achieved via adjusting the parameter $c$ based on the relation between mean exit time and quasipotential.

\begin{remark}
There are two advantages of this control strategy. On one hand, the shape of quasipotential does not change after adding this control, so the most probable path does not change. On the other hand, since the rotational component does not contribute to the mean exit time, the control varying the potential component directly costs minimum energy input.
\end{remark}
How to compute quasipotential and gradient of quasipotential?
Traditional methods of computing quasipotential and its gradient include action plot, string method and ordered upwind method (OUM).
\begin{itemize}
  \item The action plot method \cite{BMLSM} is to take a small circle around the stable fixed point. Integrating both the Hamiltonian system and Eq. \eqref{F} yields $(x(t),p(t),V(t))$. But unfortunately there are two disadvantages, one is that the quasipotential and momentum can only be calculated on the characteristic line, and consequently the direct mapping relationship between the quasipotential or momentum and the coordinates cannot be obtained; the other is that when the integrated Lagrange manifold is folded, the multiple local extrema of the action functional appear and we will obtain several quasipotential values, of which the actual quasipotential is the smallest one, i.e., the global extremum.
  \item The string method \cite{T2018} iteratively calculates the most probable path starting from a given initial path by way of fixing two endpoints. It allows us to get the quasipotential value of the endpoint at the same time. The disadvantage is that it can only be used to calculate the quasipotential of one point, and then the global quasipotential structure cannot be displayed.
  \item The OUM \cite{SMK} ensures that the quasipotential of the whole field can be obtained by discretizing the phase space into a mesh and calculating the quasipotential of the nodes point by point. However, the calculation is extremely time-consuming.
\end{itemize}
\begin{remark}
The action plot method, the string method and the OUM have their flaws. In contrast, the advantage of machine learning lies in a architecture allowing for very fast
computational and response times. The benefits of using machine learning
are well documented.
Therefore, in this section, we design a machine learning method for computing quasipotential and its gradient (i.e., momentum) using neural networks and automatic differentiation techniques \cite{RPK}.

\end{remark}

\begin{figure}[H]
	\centering
	\includegraphics[width=7cm]{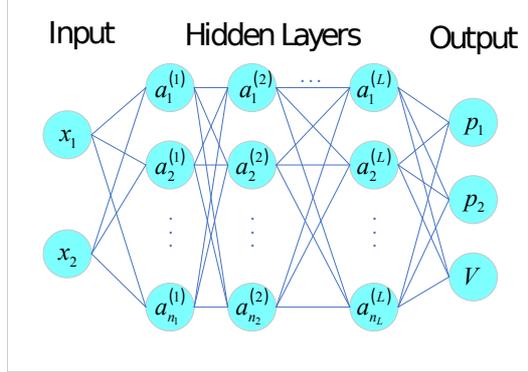}
	\caption{Architecture of the feedforward artificial neural network with $L$ hidden layers for connecting neurons. The input is the coordinate $x$, the output is the momentum $p$ and the quasipotential $V$, and $a^{l}_{j}$ denotes the value of the $j$-th neuron in $l$-th hidden layer for $j=1,\cdots,n_l$, $l=1,\cdots,L$.}
	\label{fig2}
\end{figure}

As depicted in Fig. \ref{fig2}, the schematic of the structure of the neural network is made up of three basic components: two inputs $x_1$ and $x_2$,  $L$ hidden layers composed of multiple neurons, three outputs $p_1$, $p_2$ and $V$, which are linked by edges. The activation function of the hidden layer is chosen as the hyperbolic tangent function $\tanh(x)$, and the activation function of the output layer is chosen as the identity function. It should be pointed out that the numbers of neurons in all layers do not necessarily need to be the same in the architecture.

Next, we should define the loss function for training the neural network.
According to the properties of quasipotential, the post-training quasipotential $\widetilde{V}(x)$ and momentum $\widetilde{p}(x)$ need to satisfy three constraints: $p=\nabla V(x)$, Hamilton-Jacobi equation, and the quasipotential of the stable fixed point $x_s$ being zero. We select $N_H$ collocation points on phase space to realize these constraints in the loss function \cite{LXDLC}.

However, it should be noted that the solution of the quasipotential and momentum which are identically equal to 0 always satisfies these three constraints. In order to exclude this trivial solution, we use the action plot method to integrate Eqs. (\ref{HamiltonE}) and (\ref{Vdot}), and then gain a part of the data as the fourth part of the loss function. To rule out the influence of the local extreme value of the action calculated by the action plot method due to singularity, we can adopt the following method. We discretize the phase space first and keep the grid not too dense. Then we select enough points from the small circle near the fixed point, and compute the integral to obtain a large number of extreme paths. In each grid, we collect the points of the extreme paths located in this grid and record the coordinate $x^i$, momentum $p^i$ and quasipotential $V^i$ of the point corresponding to the minimum value of the quasipotential among these points, which are taken as a set of data. The total number of the data is $N_d$.

Above all, the loss function is designed as follows. The first part of the loss function is
\begin{equation}
	\label{Lp}
	L_p=\frac{1}{N_H} \sum^{N_H}_{i=1} \left| \widetilde{p} (x^i)- \nabla \widetilde{V} (x^i)  \right| ^2.
\end{equation}
The second part of the loss function is
\begin{equation}
	\label{LH}
	L_H= \frac{1}{N_H} \sum^{N_H}_{i=1} H \left( x^i, \widetilde{p} (x^i)  \right) ^2.
\end{equation}
The third part of the loss function is
\begin{equation}
	\label{L0}
	L_0= \widetilde{V} (x_s) ^2.
\end{equation}
The fourth part of the loss function is
\begin{equation}
	\label{Ld}
	L_d=\frac{1}{N_d} \sum^{N_d}_{i=1} \left[ \left| \widetilde{p} (x^i)- p^i  \right| ^2 + \left( \widetilde{V} (x^i)- V^i  \right) ^2 \right].
\end{equation}
Therefore, the total loss function is computed as
\begin{equation}
	\label{lossall}
	L_{\text{all}}=L_p+ L_H +L_0 +L_d.
\end{equation}

The trained neural network can output the momentum and quasipotential at the specified position. Since the lowest point of the quasipotential on the boundary is usually the saddle point. We start from the point near the saddle point to calculate the time-reverse integral of the equation \eqref{Vdot}, where we use neural network calculation to get the momentum value. Eventually, we obtain the most probable path for the system to leave the attractive basin.

Next we present our control strategy to realize the aim of controlling mean exit time of stochastic systems via adjusting the parameter $c$ based on the functionality of the neural network.
After adding $u=c\nabla V(x)$ to $\dot{x}=F(x)$,
\begin{equation*}
\dot{x}=F(x)+c\nabla V(x)=-\frac{1}{2}\nabla V(x)+l(x)+c\nabla V(x)=-\frac{1}{2}(1-2c)\nabla V(x)+l(x).
\end{equation*}
For this reason, the new quasipotential is $(1-2c)V(x)$. Particularly the new quasipotential at the saddle point is $(1-2c)V_0$, where we use
 neural networks to calculate $V_0$.

Assume that our desired time is $T_d$. We notice that under the control,
\begin{equation}\label{Td}
T_d\sim\exp\Big\{\frac{(1-2c)V_0}{\sigma}\Big\},
\end{equation}
and then it follows that
\begin{equation}\label{cvalue}
c\approx\frac{1}{2}-\frac{\sigma}{2V_0}\ln T_d.
\end{equation}
Considering the influence of the prefactor $b$, we advance the algorithm by updating the value of $c$.
Given $c_1=\frac{1}{2}-\frac{\sigma}{2V_0}\ln T_d$, inserting it into \eqref{CS} gives $\dot{x}=c_1 \nabla V(x)+F(x)+\eta$.
Based on the gradient of quasipotential computed by neural network and Monte Carlo method, we simulate some sample trajectories to exit the considered domain and calculate the mean exit time $T_1$ by averaging these random exit times. Note that
\begin{equation*}
T_1=b\exp\Big\{\frac{(1-2c_1)V_0}{\sigma}\Big\}, \ \ \ T_d=b\exp\Big\{\frac{(1-2c_d)V_0}{\sigma}\Big\}.
\end{equation*}
Thus we have
\begin{equation}\label{cupdate}
 c_d=c_1+\frac{\sigma}{2V_0}\ln\frac{T_1}{T_d}.
\end{equation}
This process can be iterated continuously if the accuracy of Eq. (\ref{cupdate}) is not sufficiently satisfactory.

It needs to be emphasized that this strategy can achieve any desired mean exit time when we choose appropriate constant $c$. When $c=1/2$, the potential well of the system is zero, and the system exits quickly. When $c$ tends to negative infinity, the mean exit time goes to infinity. Therefore, our control strategy can make it possible for the mean exit time to take a certain value from small positive number to infinity.

\section{Numerical experiments}\label{N}
We now demonstrate how the update strategy of constant $c$
could help to further improve on the performance of the algorithm.
Take the Maier-Stein system \cite{MS} as an example to verify the effectiveness of the control strategy:
\begin{equation}\label{Ms}
\left\{
  \begin{array}{ll}
\dot{x}_1=x_1-x_1^{3}-\gamma x_1x_2^{2}+\sqrt{\sigma}\eta_1(t), &  \\[1ex]
\dot{x}_2=-\big(1+x_1^{2}\big)x_2+\sqrt{\sigma}\eta_2(t), &
  \end{array}
\right.
\end{equation}
where the system parameter $\gamma>0$. We are able to rewrite it in the form of \eqref{SM}, if we identify
\begin{equation*}
x:=\left(
     \begin{array}{c}
       x_1 \\
       x_2 \\
     \end{array}
   \right),\quad F(x):=\left(
     \begin{array}{c}
       x_1-x_1^{3}-\gamma x_1x_2^{2} \\
       -\big(1+x_1^{2}\big)x_2 \\
     \end{array}
   \right),\quad B:=\sqrt{\sigma}\left(
                                   \begin{array}{cc}
                                     1 & 0 \\
                                     0 & 1 \\
                                   \end{array}
                                 \right),\quad \eta(t):=\left(
     \begin{array}{c}
       \eta_1(t) \\
       \eta_2(t) \\
     \end{array}
   \right).
\end{equation*}

When $\sigma=0$, we locate the fixed points of the deterministic model $\dot{x}=F(x)$ by solving the equations $\dot{x}_1=\dot{x}_2=0$. Hence
$\dot{x}_2=0$ if $x_2=0$, and then $\dot{x}_1=0$ if $x_1-x_1^{3}=0$, which
has solutions $x_1=0$ and $x_1=\pm1$. Therefore, there are three fixed points $(0,0)$ and $(\pm1,0)$.
Linearize by finding the Jacobian matrix:
\begin{equation*}
 J=\left(\begin{array}{cc}
     1-3x_1^2-\gamma x_2^2 & -2\gamma x_1x_2 \\
     -2x_1x_2 & -(1+x_1^2)
   \end{array}\right).
\end{equation*}
Linearize at the origin:
\begin{equation*}
 J_{(0,0)}=\left(\begin{array}{cc}
     1 & 0 \\
     0 & -1
   \end{array}\right).
\end{equation*}
There is one positive and one negative eigenvalue, and so this fixed point is a
saddle point.
 For other fixed points,
\begin{equation*}
 J_{(\pm1,0)}=\left(\begin{array}{cc}
     -2 & 0 \\
     0 & -2
   \end{array}\right).
\end{equation*}
There are two distinct negative eigenvalues and hence
the fixed points are stable nodes.

Note that the matrices $J_{(0,0)}$ and $J_{(\pm1,0)}$ are in diagonal form. The eigenvectors
for both fixed points are $(1,0)^{T}$ and $(0,1)^{T}$. Thus in a small neighborhood
around each fixed point, the stable and unstable manifolds are tangent to the
lines generated by the eigenvectors through each fixed point. Therefore, near each
fixed point the manifolds are horizontal and vertical. The manifolds of the nonlinear system
$W_s$ and $W_u$ need not be straight lines but are tangent to the subspaces $E_s$ and $E_u$ at the relevant fixed point.

Consider the isoclines. Now $\dot{x}_2=0$ on $x_2=0$, and on this line $\dot{x}_1=x_1-x_1^{3}$. Thus
if $x_1>1$ or $-1<x_1<0$, then $\dot{x}_1<0$, and if $x_1<-1$ or $0<x_1<1$, then $\dot{x}_1>0$. Also, $\dot{x}_1=0$ on the circle $\dot{x}_1^2+\gamma\dot{x}_2^2=1$ with $\gamma>0$, and on this curve $\dot{x}_2=(\gamma\dot{x}_2^2-2)x_2$. Thus if $x_2>\sqrt{2/\gamma}$ or $-\sqrt{2/\gamma}<x_2<0$, then $\dot{x}_2> 0$, and if $x_2<-\sqrt{2/\gamma}$ or $0<x_2<\sqrt{2/\gamma}$, then $\dot{x}_2<0$.
 The slope of the trajectories is given by
\begin{equation*}
 \frac{dx_2}{dx_1}=\frac{-\big(1+x_1^{2}\big)x_2}{x_1-x_1^{3}-\gamma x_1x_2^{2}}.
\end{equation*}

\begin{figure}
	\centering
	\subfigure[$\gamma=1$]
	{\includegraphics[width=7cm]{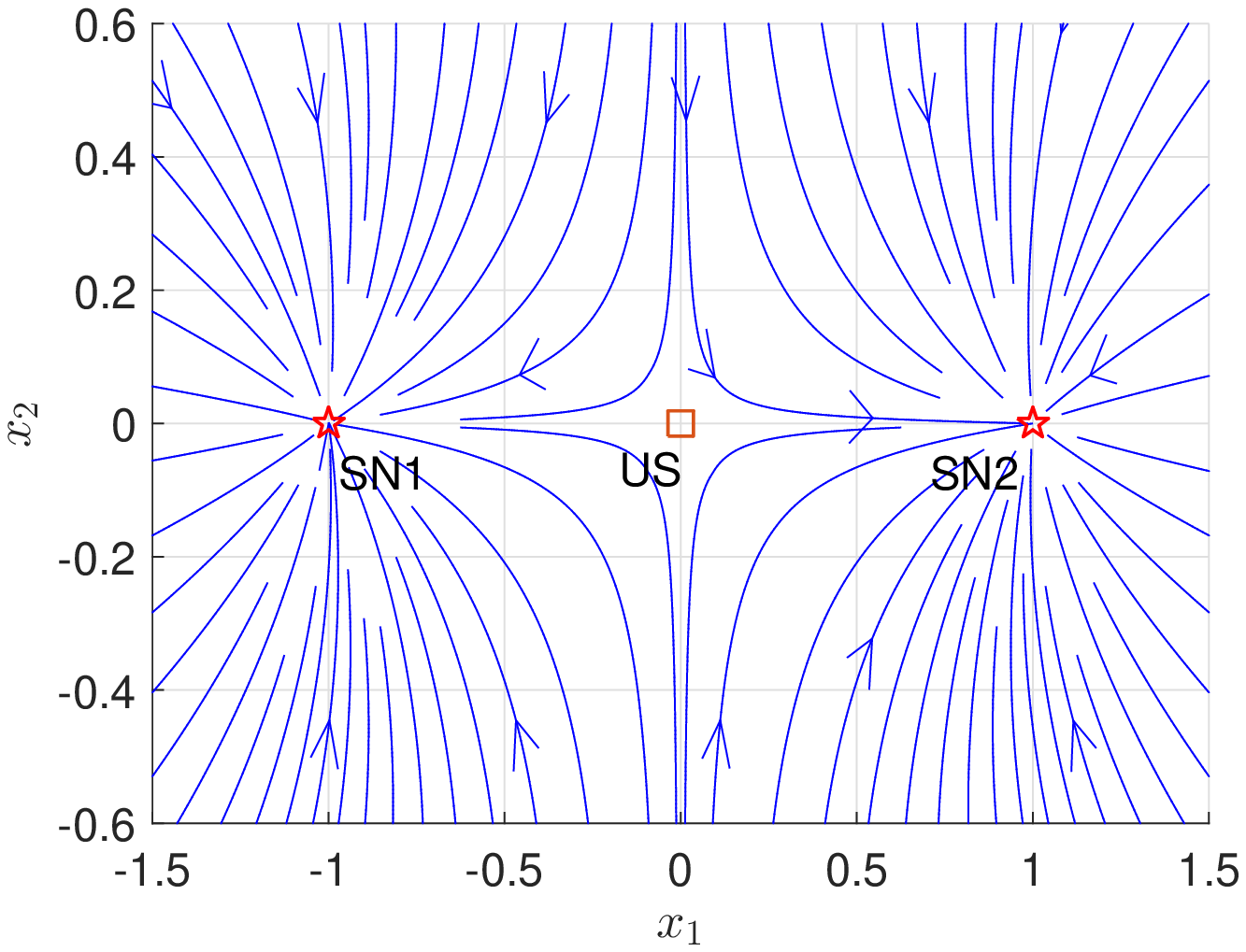}}
	\subfigure[$\gamma=5$]
	{\includegraphics[width=7cm]{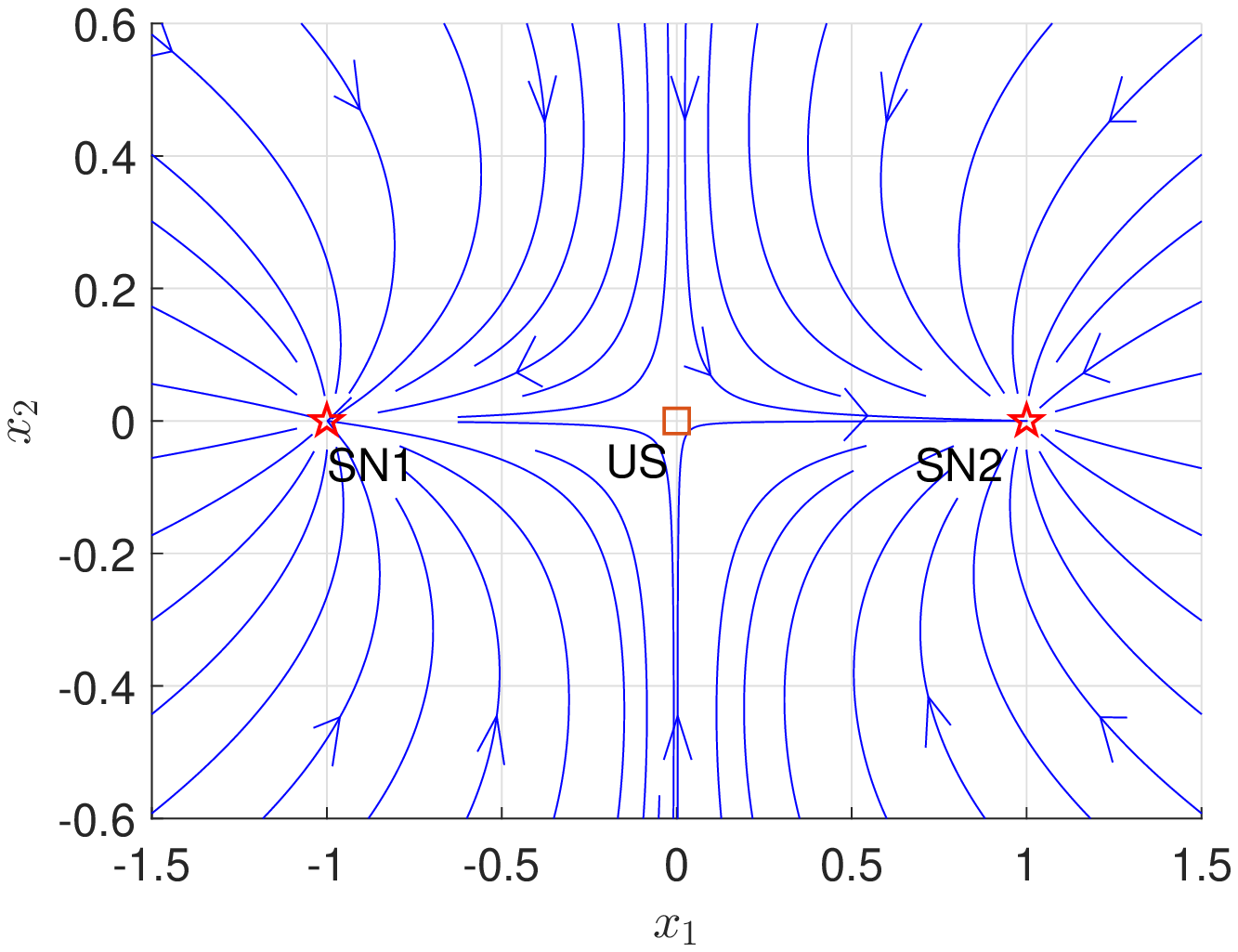}}
	\caption{Vector fields of Maier-Stein system for $\gamma=1$ and $\gamma=5$, where SN1 and SN2 denote two stable nodes and US indicates a unstable saddle.}
	\label{fig3}
\end{figure}

As can be seen from Fig. \ref{fig3}, we clearly illustrate the structure of the deterministic system $\dot{x}=F(x)$ using phase portraits with US at $(0,0)$, SN1 at $(-1,0)$ and SN2 at $(1,0)$. Due to the left-right symmetry of the system, we only need to consider the case of escaping from the basin of SN1 to the right.

Parameter settings allow us to customize various values to suit our needs.
The neural network is composed of 4 hidden layers, each layer has 20 neurons, the optimizer selects Adam, the learning rate is 0.02, and the training contains 50,000 steps. There are 5000 points selected randomly in the area $[-1.5,0]\times[-0.6,0.6]$, i.e., $N_H=5000$. Divide this area into a $20*20$ grid, i.e.,
$N_d=400$.

Utilizing the action plot method, 2000 starting points are uniformly selected on the small circle near the fixed point, and 2000 parameterized trajectories are obtained by integration. The coordinate, momentum and quasipotential of the point with the smallest quasipotential falling in each grid are extracted, and then 400 groups of data are collected. The control increases or decreases the mean exit time.
The next step is to consider two cases of $\gamma=1$ and $\gamma=5$ respectively.

\textbf{Case A}. $\gamma=1$ corresponds to the gradient system or equilibrium system with a potential function
\begin{equation*}
	U(x)=\frac{1}{4} \left[ \left( x^2_1 -1 \right) ^2 + 2x^2_2 \left( x^2_1 +1 \right) \right].
\end{equation*}
According to Freidlin-Wentzell large deviation theory, the quasipotential in this case is exactly twice of the potential function, i.e., $V(x)=2U(x)$.

The loss function after training is reduced to the order of $10^{-5}$. Fig. \ref{fig4} shows a comparison of Learned and True quasipotentials, which are very consistent. It demonstrates how Learned quasipotential may be used to compute True quasipotential in a manner entirely consistent with modern architectures.
\begin{figure}[H]
	\centering
	\subfigure[True quasipotential]
	{\includegraphics[width=7cm]{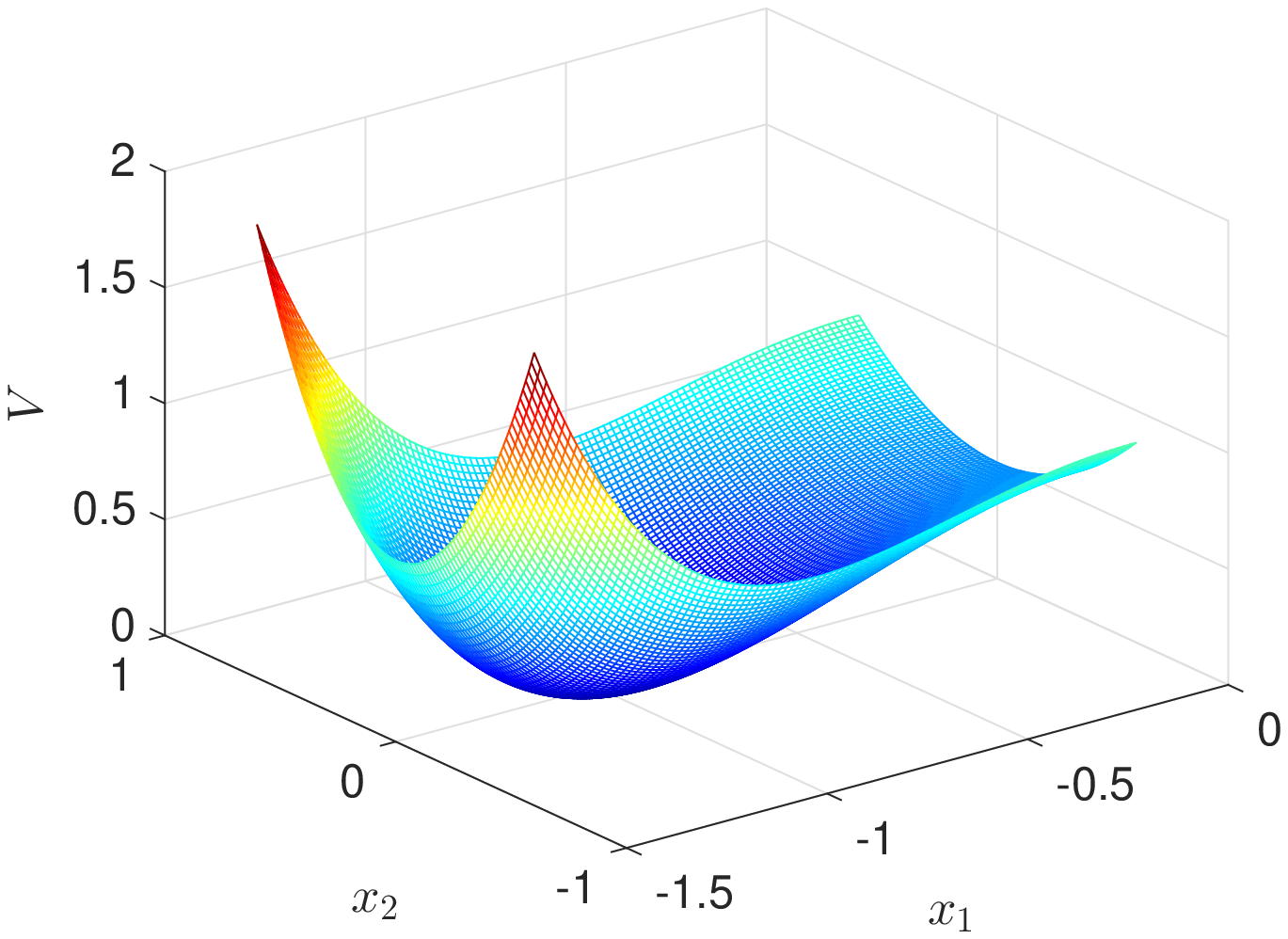}}
	\subfigure[Learned quasipotential]
	{\includegraphics[width=7cm]{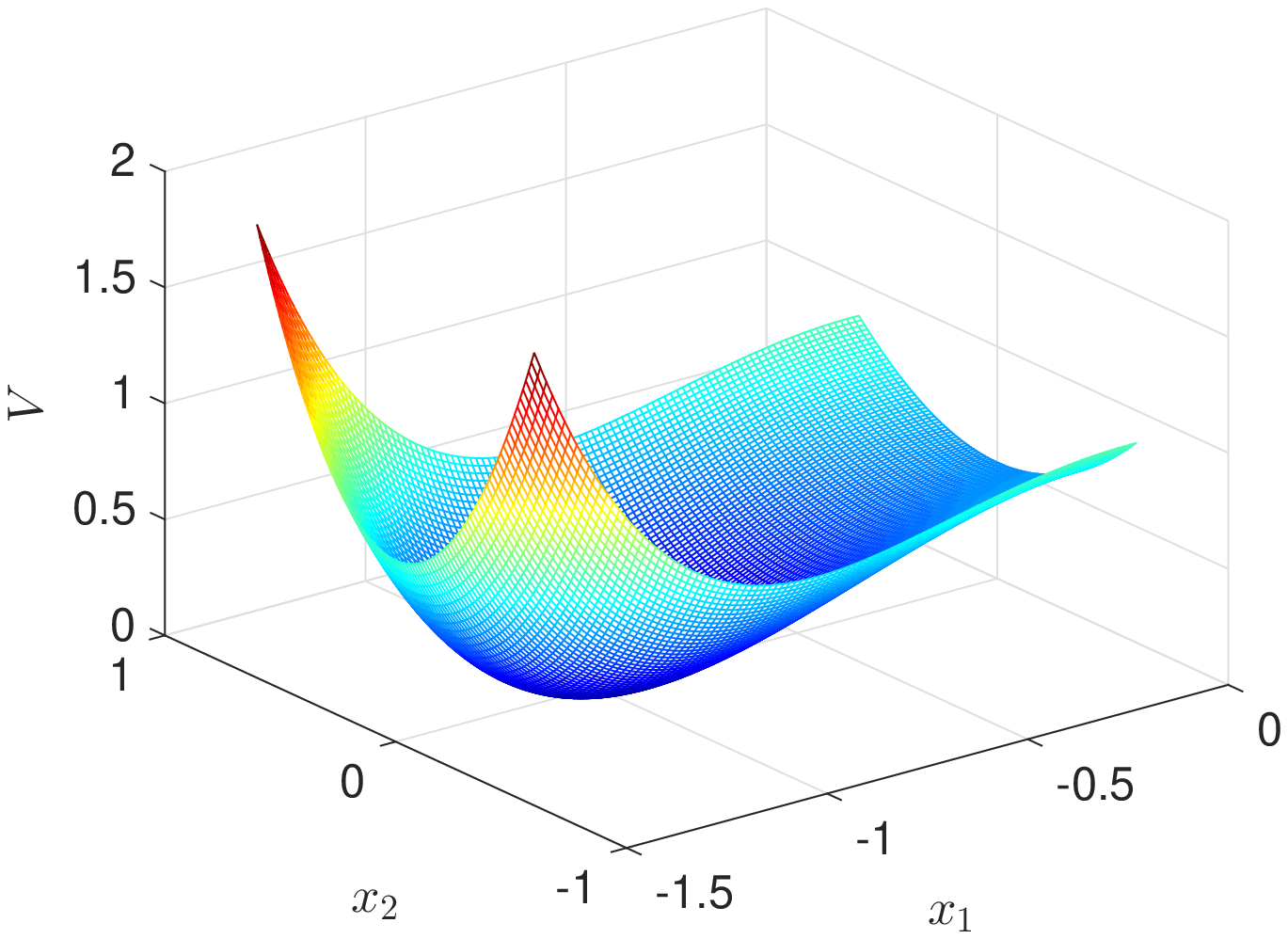}}
	\caption{Comparison between Learned and True quasipotential for $\gamma=1$.}
	\label{fig4}
\end{figure}

The most probable path is  found by reverse-time integration of Eq. (\ref{xdot}) starting from a point near the left of the saddle point $(0,0)$, as shown in Fig. \ref{fig5}. Since the potential system satisfies the detailed balance condition, the most probable path starting from the stable fixed point $(-1,0)$ to the saddle UN is exactly the time reverse of the deterministic trajectory starting from the saddle to the fixed point, which is the horizontal line. It is seen that our experiment results are consistent with the theoretical ones.
\begin{figure}[H]
	\centering
	\includegraphics[width=7cm]{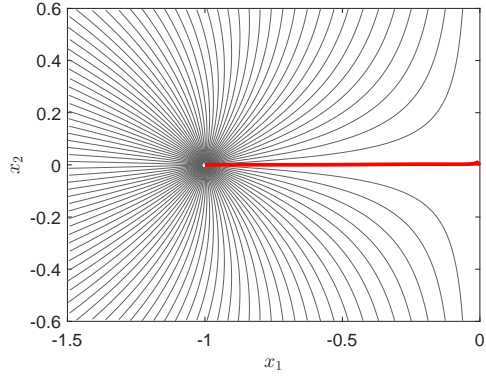}
	\caption{The most probable path denoted by red line and extreme paths for $\gamma=1$.}
	\label{fig5}
\end{figure}

\begin{table}[htbp]
	\centering
	\caption{Controlling mean exit time for $\gamma=1$}
	\begin{tabular}{ccccccc}
		\toprule
		\multicolumn{1}{c}{$\sigma$} & \multicolumn{1}{c}{$T_0$} & \multicolumn{1}{c}{$T_d$} & \multicolumn{1}{c}{$c_1$} & \multicolumn{1}{c}{$T_{c_1}$} & \multicolumn{1}{c}{$c_2$} & \multicolumn{1}{c}{$T_{c_2}$} \\
		\midrule
		0.15 & 58.08  & 100 & -0.1901 & 136.94 & -0.1430 & 104.43 \\
		0.1 & 270.43  & 100 & 0.0399 & 202.07 & 0.1102 & 115.74 \\
		0.05 & 39157.98  & 100 & 0.2700 & 417.52 & 0.3414 & 146.26 \\
		0.15 & 58.08  & 1000 & -0.5351 & 877.87 & -0.5546 & 950.71 \\
		0.1 & 270.43  & 1000 & -0.1901 & 1296.43 & -0.1642 & 1014.83 \\
		0.05 & 39157.98  & 1000 & 0.1550 & 2645.23 & 0.2035 & 1191.72 \\
		\bottomrule
	\end{tabular}
	\label{table1}
\end{table}
Next we examine our control strategy and describe the process of the algorithm.
From the list items in Table \ref{table1}, we choose three noise intensities $\sigma=0.15, 0.1$ and $0.05$ such that the uncontrolled mean exit time $T_0=58.08$ is less than 100, $T_0=270.43$ is greater than 100 but less than 1000, and $T_0=39157.98$ is far greater than 1000. To estimate the mean exit time, we simulate 1000 random  trajectories and calculate the expected value.
When the desired mean exit time is $T_d=100$,
we have $T_{c_1}=136.94, 202.07$ and $417.52$ if we set $c_1=-0.1901, 0.0399$ and $0.2700$ respectively. We can make far more precise correction to get $T_{c_2}=104.43, 115.74$ and $146.26$ with $c_2=-0.1430, 0.1102$ and $0.3414$. It is worthy to note that $T_{c_2}$ is closer to $T_d=100$ than $T_{c_1}$ for all the three cases.
When the desired mean exit time is $T_d=1000$, we get $T_{c_1}=877.87, 1296.43$ and $2645.23$ if we set $c_1=-0.5351, -0.1901$ and $0.1550$  respectively.
Furthermore, we obtain $T_{c_2}=950.71, 1014.83$ and $1191.72$ with $c_2=-0.5546, -0.1642$ and $0.2035$. After a correction,  the mean exit time is almost as big as the desired value $T_d=1000$. If the accuracy of the correction is not enough, such as in the case of  $\sigma=0.05$, $T_d=100$, multiple corrections can be further performed by taking $c_3=0.3604$ to gain $T_{c_3}=113.70$. Therefore, our control strategy can customize the mean exit time to reach the desired value.

\textbf{Case B}. $\gamma=5$ corresponds to non-gradient system or non-equilibrium system without a potential function.

The loss function after training is reduced to the order of $10^{-4}$. Learned quasipotential for $\gamma=5$ is plotted in Fig. \ref{fig6}. The most probable path is obtained by reverse-time integration starting from a point near the left of the saddle point, as depicted in Fig. \ref{fig7}. In this case, the most probable path bifurcates into two black upper and lower paths symmetrically because of the singularity of the Lagrangian manifold.
\begin{figure}
	\centering
	\includegraphics[width=7cm]{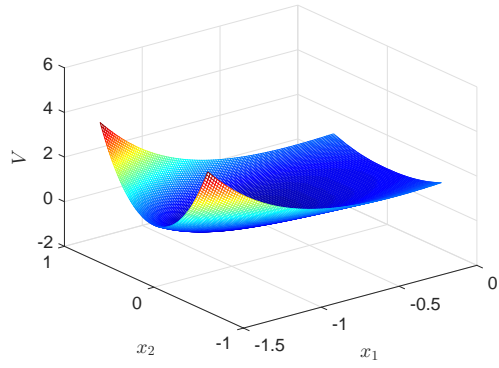}
	\caption{Learned quasipotential for $\gamma=5$.}
	\label{fig6}
\end{figure}

\begin{figure}
	\centering
	\includegraphics[width=7cm]{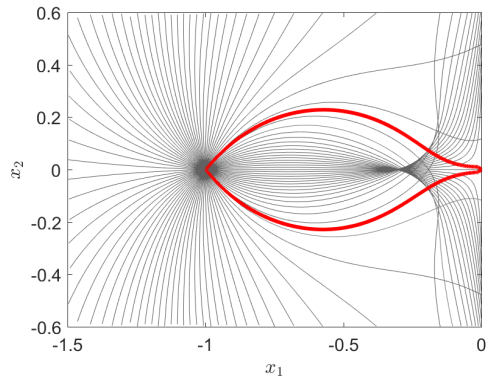}
	\caption{The most probable paths denoted by red line and extreme paths for $\gamma=5$.}
	\label{fig7}
\end{figure}
Now we examine our control strategy again and specify the content of Table \ref{table2}.  We select three noise intensities $\sigma=0.1, 0.075$ and $0.05$, such that the quantities of the uncontrolled mean exit time are $T_0=70.52, 306.06$ and $6244.60$. When the desired mean exit time is $T_d=100$,
we have $T_{c_1}=65.04, 71.19$ and $81.60$ if $c_1=0.0241, 0.1430$ and $0.2620$. We can make far more precise correction to get $T_{c_2}=83.54, 88.76$ and $96.38$ with $c_2=-0.0204, 0.1167$ and $0.2515$. It is worth pointing out that $T_{c_2}$ is closer to $T_d=100$ than $T_{c_1}$.
When the desired mean exit time is $T_d=1000$, we get $T_{c_1}=404.54, 446.47$ and $519.54$ if $c_1=-0.2139, -0.0354$ and $0.1430$.
After a correction, we obtain $T_{c_2}=911.13, 884.27$ and $889.52$ with $c_2=-0.3074, -0.0979$ and $0.1092$. The mean exit time basically is close to the desired value. Similarly, if the accuracy of the correction is not enough, such as in the case of $\sigma=0.075$, $T_d=1000$, multiple corrections can be performed to acquire $T_{c_3}=1007.68$  with $c_3=-0.1075$. As a result, the mean exit time will be getting much nearer to the desired value.

\begin{table}[htbp]
	\centering
	\caption{Controlling mean exit time for $\gamma=5$}
	\begin{tabular}{ccccccc}
		\toprule
		\multicolumn{1}{c}{$\sigma$} & \multicolumn{1}{c}{$T_0$} & \multicolumn{1}{c}{$T_d$} & \multicolumn{1}{c}{$c_1$} & \multicolumn{1}{c}{$T_{c_1}$} & \multicolumn{1}{c}{$c_2$} & \multicolumn{1}{c}{$T_{c_2}$} \\
		\midrule
		0.1 & 70.52  & 100 & 0.0241 & 65.04 & -0.0204 & 83.54 \\
		0.075 & 306.06  & 100 & 0.1430 & 71.19 & 0.1167 & 88.76 \\
		0.05 & 6244.60  & 100 & 0.2620 & 81.60 & 0.2515 & 96.38 \\
		0.1 & 70.52  & 1000 & -0.2139 & 404.54 & -0.3074 & 911.13 \\
		0.075 & 306.06  & 1000 & -0.0354 & 446.47 & -0.0979 & 884.27 \\
		0.05 & 6244.60  & 1000 & 0.1430 & 519.54 & 0.1092 & 889.52 \\
		\bottomrule
	\end{tabular}
	\label{table2}
\end{table}

\section{Conclusion}\label{CF}
In this paper, we developed a feedforward multilayer network with hidden layers for stimulating the mean exit time that is dominant in helping a process escape a bounded domain. Meanwhile, we devised a machine learning framework to compute the quasipotential and most probable paths of metastable transition events. We improved the control strategy to reduce the error between the output vector and the target vector. A typical architecture was shown in Fig. \ref{fig2}. This work generated a huge amount of interest. Pragmatically, we illustrated the effectiveness of the control strategy with an example about the Maier-Stein system.

We can generalize this two-dimensional case by further considering the rather cumbersome/difficult case of general multi-dimensional or time-dependent vector fields. We even suggest a possible treatment in high dimensions by looking for a low-dimensional manifold near which the dynamics are focused. In many variable situations, stochastic neural network can be set up as a type of artificial neural network built by introducing random variations into the network, either by giving the network's artificial neurons stochastic transfer functions, or by giving them stochastic weights. This makes them become useful tools for optimization problems, since the random fluctuations help the network escape from local minima. The expectations are high for future applications in a broad range of disciplines.

\bigskip
\noindent\textbf{Acknowledgements}

The authors are happy to thank Haitao Xu for fruitful discussions on Hamiltonian systems.

\bigskip
\noindent\textbf{Funding}

The authors acknowledge support from the natural science foundation of Jiangsu Province (grant BK20220917).

\bigskip

\noindent\textbf{Data Availability Statement}

Numerical algorithms source code that support the findings of this study are openly
available in GitHub, Ref \cite{code}.

\bigskip

\noindent\textbf{Conflict of Interest}

The authors declare that they have no conflict of interest.

   % None declared.
%\label{}

%This paper is organized as follows. In section 2, we present the framework for our random slow manifold reduction method for parameter estimation. Then in Section 3, we establish the error estimation for our parameter estimator in terms of observation error and slow reduction error by using analytical technique. Finally, we illustrate our estimation method numerically in a specific example.

%a specific example is tested to illustrate our estimation method in Section 4This paper is organized as follows. In section 2, we obtain an approximated random slow manifold and thus the random slow system. Then in Section 3, we establish the error estimation for our parameter estimator in terms of observation error and slow reduction error by using analytical technique. Finally, .


%%       \citet{<label>} ==> Jones et al. [21]
%%       \citep{<label>} ==> [21]
%%



\end{document}